%% file: main.tex
\documentclass[11pt]{article} 
\usepackage{url}
\usepackage{smile}
\usepackage{graphicx} 
\usepackage{algorithm}
\usepackage{algorithmic}
\usepackage{todonotes}
\usepackage{epstopdf}
\usepackage[colorlinks, linkcolor=blue, anchorcolor=blue, citecolor=blue]{hyperref}
\usepackage[margin=1in]{geometry}
\usepackage[normalem]{ulem}
\usepackage[export]{adjustbox}
\usepackage{mathtools, cuted}
\usepackage{natbib}

\usepackage{todonotes}
\newcommand{\Note}[2]{}

\usepackage{kpfonts}
\DeclareMathAlphabet{\mathsf}{OT1}{cmss}{m}{n}
\SetMathAlphabet{\mathsf}{bold}{OT1}{cmss}{bx}{n}

\newcommand{\removed}[1]{}

\newcommand{\yestwocl}[1]{
\ifdefined\twocl 
{#1} 
\fi}
\newcommand{\notwocl}[1]{
\ifdefined\twocl 
{}
\else 
{#1} 
\fi}

\title{\Huge \bf The Physical Systems Behind Optimization Algorithms}

\author{L. F. Yang,  R. Arora, V. Braverman, and T. Zhao\thanks{Lin Yang is affiliated with Department of Operations Research and 　Financial Engineering at Princeton University; Raman Arora and Vladimir Braverman are affiliated with Department of Computer Science at Johns Hopkins University; Tuo Zhao is affiliated with School of Industrial and Systems Engineering at Georgia Institute of Technology; Email: \texttt{lin.yang@princeton.edu, tourzhao@gatech.edu}; Tuo Zhao is the corresponding author.}}

\date{}

\begin{document}

\maketitle

\input{abstract}
\input{intro}

\input{optimization_alg}
\input{alg_ode}
\input{applications}
\input{extension}
\input{discussion}
\bibliographystyle{ims}
\bibliography{ref}

\appendix

\input{appendix}

\end{document}

%% file: abstract.tex

\begin{abstract}
We use differential equations based approaches to provide some {\it \textbf{physics}} insights into analyzing the dynamics of popular optimization algorithms in machine learning. In particular, we study gradient descent, proximal gradient descent, coordinate gradient descent, proximal coordinate gradient, and Newton's methods as well as their Nesterov's accelerated variants in a unified framework motivated by a natural connection of optimization algorithms to physical systems. Our analysis is applicable to more general algorithms and optimization problems {\it \textbf{beyond}} convexity and strong convexity, e.g. Polyak-\L ojasiewicz and error bound conditions (possibly nonconvex).
\end{abstract}

%% file: intro.tex

\section{Introduction}

Many machine learning problems can be cast into an optimization problem of the following form:
\begin{align}\label{general-problem}
x^*=\argmin_{x\in\mathcal{X}} f(x),
\end{align}
where $\mathcal{X}\subseteq\RR^d$ and $f:\mathcal{X}\rightarrow \RR$ is a continuously 
differentiable function. For simplicity, we assume that $f$ is convex or approximately convex (more on this later). Perhaps, the earliest algorithm for solving \eqref{general-problem} is the vanilla gradient descent (VGD) algorithm, which dates back to Euler and Lagrange. VGD is simple, intuitive, and easy to implement in practice. For large-scale problems, it is usually more scalable than more sophisticated algorithms (e.g. Newton).

Existing state-of-the-art analysis of VGD shows it achieves $\cO(1/k)$ convergence for general convex functions and linear convergence rate for strongly convex functions, where $k$ is the number of iterations \citep{Nes13}. Recently, a class of so-called Nesterov's accelerated gradient (NAG) algorithms has gained popularity in statistical signal processing and machine learning communities. These algorithms combine the vanilla gradient descent algorithm with an additional momentum term at each iteration. Such a modification, though simple, has a profound impact: the NAG algorithms attain faster convergence than VGD. Specifically, NAG achieves $\cO(1/k^2)$ convergence for general convex functions, and linear convergence with a better constant term for strongly convex functions \citep{Nes13}.

Another closely related class of algorithms is randomized coordinate gradient descent (RCGD) algorithms. These algorithms conduct a gradient descent-type step in each iteration, but only with respect to a single coordinate. RCGD has similar convergence rates to VGD, but has a smaller overall computational complexity, since its computational cost per iteration of RCGD is much smaller than VGD~\citep{nesterov2012efficiency,lu2015complexity}. More recently, \cite{lin2014accelerated,fercoq2015accelerated} applied Nesterov's acceleration to RCGD, and proposed accelerated randomized coordinate gradient (ARCG) algorithms. Accordingly, they established similar accelerated convergence rates for ARCG.



Another line of research focuses on relaxing the convexity and strong convexity conditions 
for alternative regularity conditions, including restricted secant inequality, error bound, Polyak-\L ojasiewicz, and quadratic growth conditions. These conditions have been shown to hold for many optimization problems in machine learning, and faster convergence rates have been established (e.g. \cite{luo1993error, liu2015asynchronous, necoara2015linear, zhang2013gradient, gong2014linear,
KNS16}). 




Although various theoretical results have been established, the algorithmic proof of convergence and regularity conditions in these analyses rely heavily on algebraic tricks that are sometimes arguably mysterious to be understood. To address this concern, differential equation approaches recently have attracted enormous 
interests on the analysis of optimization algorithms, because they provide a clear interpretation for the continuous approximation of the algorithmic systems \citep{SBC14,WWJ16}.
 \cite{SBC14} propose a framework for studying discrete 
 algorithmic systems under the limit of infinitesimal time step. 
 They show that Nesterov's accelerated gradient (NAG) algorithm
 can be described by an ordinary differential equation (ODE)
 under the limit that time step tends to $0$. \cite{WWJ16} study a more general family of 
 ODE's that essentially correspond to accelerated gradient algorithms.
 All these analyses, however, lack a link to a natural physical system behind the optimization algorithms. Therefore, they do not clearly explain why the momentum leads to
 acceleration. Meanwhile, these analyses only consider general convex conditions and gradient descent-type algorithms, and are NOT applicable to either the aforementioned relaxed conditions or coordinate-gradient-type algorithms (due to the randomized coordiante selection).
 
 
\noindent{\bf{Our Contribution (I):}}
We provide some new {\bf physics} insights for the differential equation approaches. Particularly, we connect the differential equations to natural physical systems.  Such a new connection allows us to establish a unified theory for understanding these optimization algorithms. Specifically, we consider the VGD, NAG, RCGD, and ARCG algorithms. All these algorithms are associated with {\bf damped oscillator} 
 systems with different {\bf particle mass} and {\bf damping coefficients}.
 For example, VGD corresponds to a massless particle system, NAG corresponds to a massive particle system. 
 A damped oscillator system has a natural dissipation of its mechanical energy. The decay rate of the mechanical energy 
 in the system essentially connects to the convergence rate of the algorithm. 
 Our results match the convergence rates of all considered algorithms in existing literature.
 For a massless system, the convergence rate only depends
 on the gradient (force field) and smoothness of the function, whereas 
 a massive particle system has an energy decay rate proportional
 to the ratio between the mass and damping coefficient. 
We further show that the optimal algorithm such as NAG correspond to an oscillator system near {\bf critical damping}. Such a phenomenon is known in the physical literature that the critically damped system undergoes the fastest energy dissipation.
Thus, this approach can potentially help us to design new optimization algorithms in a more 
intuitive way. 
{As pointed out by the anonymous reviewers, although some of the intuitions provided in this paper are also presented in \cite{polyak1964some}, 
we provide a more detailed analysis in this paper.}
  
  
\noindent{\bf{Our Contribution (II):}} We provide new analysis for more general optimization problems beyond general convexity and strong convexity, as well as more general algorithms. Specifically, we provide several concrete examples: (1) VGD achieves linear convergence under the Polyak-\L ojasiewicz (PL) condition (possibly nonconvex), which matches the state-of-art result in \cite{KNS16}; (2) NAG achieves accelerated linear convergence (with a better constant term) under both general convex and quadratic growth conditions, which matches the state-of-art result in \cite{zhang2016new}; (3) Coordinate-gradient-type algorithms share the same ODE approximation with gradient-type algorithms, and our analysis involves a \emph{more refined} infinitesimal analysis; (4) Newton algorithm achieves linear convergence under the strongly convex and self-concordance conditions. See a summary in Table \ref{summary} as follows. Due to space limit, we present the extension to the nonsmooth composite optimization problem in Appendix.

\begin{table}[H]
	\centering
\caption{Our contribution compared with \cite{SBC14, WWJ16}.}
\includegraphics[width=0.7\textwidth]{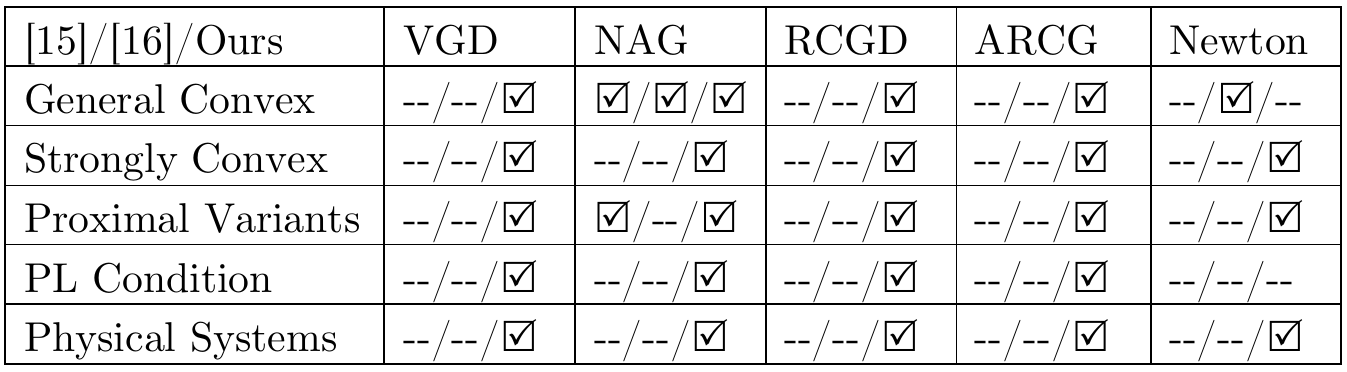}\label{summary}
\end{table}

Recently, an independent paper of \cite{wilson2016lyapunov}  studies a similar framework for analyzing first order optimization algorithms, and they focus on bridging the gap between discrete algorithmic analysis and continuous approximation. While we focus on understanding the physical systems~ behind~ the~ optimization. Both perspectives are essentially complementary to each other.

%% file: optimization_alg.tex
Before we proceed, we first introduce assumptions on the objective $f$.


\begin{assumption}[($L$-smooth)]
There exits a constant $L$ such that for any $x,~y\in\RR^d$, we have $\norm{\nabla f(x) - \nabla f(y)}\le L\norm{x-y}$.
\end{assumption}


\begin{assumption}[($\mu$-strongly convex)] There exits a constant $\mu$ such that for any $x,~y\in\RR^d$, we have $f(x) \geq f(y) + \langle\nabla f(y),x-y\rangle + \frac{\mu}{2}\norm{x-y}^2$.
\end{assumption}


\begin{assumption} ($L_{\max}$-coordinate-smooth) There exits a constant $L_{\max}$ such that for any $x,~y\in\RR^d$, we have $|\nabla_jf(x) - \nabla_jf(x_{\setminus j}, y_j)|\le L_{\max}(x_j-y_j)^2$ for all $j=1,...,d$.
\end{assumption}


The $L_{\max}$-coordinate-smooth condition has been shown to be satisfied by many machine learning problems such as Ridge Regression and Logistic Regression. For convenience, we define $\kappa=L/\mu~\textrm{and}~\kappa_{\max}=L_{\max}/\mu$. Note that we also have $L_{\max}\leq L\leq dL_{\max}~\textrm{and}~\kappa_{\max}\leq \kappa\leq d\kappa_{\max}$.

%% file: alg_ode.tex

\section{From Optimization Algorithms to ODE}\label{sec:phys}

We develop a unified representation for the continuous approximations of the aforementioned optimization algorithms. Our analysis is inspired by \cite{SBC14}, where the NAG algorithm for general convex function is approximated by an ordinary differential equation under the limit of infinitesimal time step. We start with VGD and NAG, and later we will show that RCGD and ARCG can be approximated by the same ODE. For self-containedness, we present a brief review for popular optimization algorithms in Appendix~\ref{appendix:optimization algs} (VGD, NAG, RCGD, ARCG, and Newton). 

\subsection{A Unified Framework for Continuous Approximation Analysis}\label{sec:unified framework}

By considering an infinitesimal step size, we rewrite VGD and NAG in the following generic form:
\begin{align}
\label{eqn:unified-gradient}
x^{(k)} = y^{(k-1)} - \eta \nabla f(y^{(k-1)}) 
\ifdefined\twocl
\nonumber\\
\else
\quad \textrm{and}\quad
\fi
y^{(k)} = x^{(k)} + \alpha (x^{(k)} - x^{(k-1)}).
\end{align}
For VGD, $\alpha=0$; For NAG, $\alpha= \frac{\sqrt{1/(\mu\eta)}-1}{\sqrt{1/(\mu\eta)}+1}$ when $f$ is strongly convex, and $\alpha=\frac{k-1}{k+2}$ when $f$ is general convex. 
We then rewrite \eqref{eqn:unified-gradient} as
\begin{align}
\label{eqn:discrete-eqn}
\rbr{x^{(k+1)} - x^{(k)}} -\alpha \rbr{x^{(k)}-x^{(k-1)}} 
\ifdefined\twocl
\nonumber\\
\else
\fi
+ \eta \nabla f\rbr{x^{(k)} + \alpha (x^{(k)} -x^{(k-1)})} = 0.
\end{align}
{When considering the continuous-time limit of the above equation, it is not immediately clear how the continuous-time is related to the step size $k$. We thus
let $h$ denote the time scaling factor and study the possible choices of $h$ later on.} With this, we define a continuous time variable
\begin{align}\label{time-iteration-scaling}
t=kh\quad\textrm{with}\quad X(t) = x^{(\lceil t/h\rceil)} = x^{(k)},
\end{align} 
where $k$ is the iteration index, and $X(t)$ from $t=0$ to $t=\infty$ is a trajectory characterizing the dynamics of the algorithm.
Throughout the paper, we may omit $(t)$ if it is clear from the context. 



Note that our definition in \eqref{time-iteration-scaling} is very different from \cite{SBC14}, where $t$ is defined as $t=k\sqrt{\eta}$, i.e., {fixing $h=\sqrt{\eta}$}. 
There are several advantages by using our new definition: {\bf (1)} The new definition leads to a unified analysis for both VGD and NAG. Specifically, if we follow the same notion as \cite{SBC14}, we need to redefine $t=k\eta$ for VGD, which is different from $t=k\sqrt{\eta}$ for NAG; {\bf (2)} The new definition is more flexible, and leads to a unified analysis for both gradient-type (VGD and NAG) and coordinate-gradient-type algorithms (RCGD and ARCG), regardless of their different step sizes, e.g $\eta=1/L$ for VGD and NAG, and $\eta=1/L_{\max}$ for RCGD and ARCG; {\bf (3)} The new definition is equivalent to \cite{SBC14} only when $h=\sqrt{\eta}$. We will show later that, however, $h\asymp \sqrt{\eta}$ is a natural requirement of a massive particle system rather than an artificial choice of $h$.


We then proceed to derive the differential equation for \eqref{eqn:discrete-eqn}. By Taylor expansion
\begin{align*}
&\rbr{x^{(k+1)}-x^{(k)}}  = \dot{X}(t)h + \frac{1}{2}\ddot{X}(t)h^2 + o(h),\nonumber \\
&\rbr{x^{(k)}-x^{(k-1)}} = \dot{X}(t) h - \frac{1}{2}\ddot{X}(t)h^2 + o(h),\nonumber \\
~\textrm{and}~ &\eta\nabla f\sbr{x^{(k)} + \alpha\rbr{x^{(k)}-x^{(k-1)}}}
=\eta \nabla f(X(t)) + O(\eta h).
\end{align*}
where $\dot{X}(t)=\frac{dX(t)}{dt}$ and $\ddot{X}(t) =\frac{d^2X}{dt^2}$,
we can rewrite \eqref{eqn:discrete-eqn} as
\begin{align}
\label{eqn:total-eqn}
\frac{(1+\alpha)h^2}{2\eta}\ddot{X}(t)+ \frac{(1-\alpha)h}{\eta}\dot{X}(t) +  \nabla f(X(t)) +O(h) = 0.
\end{align}
Taking the limit of $h\rightarrow 0$, we rewrite \eqref{eqn:total-eqn} 
in a more convenient form,
\begin{align}\label{unified-ODE}
 m\ddot{X}(t)+c\dot{X}(t) +  \nabla f(X(t)) = 0.
\end{align}
Here \eqref{unified-ODE} describes exactly a \emph{damped oscillator system}
in $d$ dimensions with
\begin{align*}
\begin{array}{ll}
 m:=\frac{1+\alpha}{2}\frac{h^2}{\eta}\quad &\textrm{as\quad the {\it particle mass},}\\
 c:=\frac{(1-\alpha)h}{\eta}\quad &\textrm{as\quad the {\it damping coefficient},}\\  
\textrm{and}~~f(x)\quad &\textrm{as\quad the {\it potential field}}.
\end{array}
\end{align*}
Let us now consider how to choose $h$ for different settings. 
The basic principle is that \textbf{ both $m$ and $c$ are finite under the limit $h, \eta\rightarrow 0$}.
In other words, the physical system is valid.
Taking VGD as an example, for which we have $\alpha = 0$.
In this case, the only valid setting is $h=\Theta(\eta)$, under which, $m\rightarrow0$ and $c\rightarrow c_0$ for some constant $c_0$.
We call such a particle system \emph{massless}.
For NAG, it can also be verified that only  $h=\Theta(\sqrt{\eta})$ results in a \emph{valid}  physical system and it is \emph{massive} ($0<m< \infty, 0\le c<\infty$).
Therefore, we provide a unified framework of choosing the correct time scaling factor $h$.



\subsection{A Physical System: Damped Harmonic Oscillator}
In classic mechanics, the harmonic oscillator is one of the first mechanic systems, which admit an exact solution. This system consists of a massive particle and restoring force. A typical example is a massive particle connecting to a massless spring.

The spring always tends to stay at the equilibrium position. 
When it is stretched or compressed, there will be a force acting on the object that stretches or compresses it. 
The force is always pointing toward the equilibrium position.
The energy stored in the spring is
\begin{align*}
 V(X):=\frac{1}{2}\cK X^2,
\end{align*}
where $X$ denotes the displacement of the spring, and $\cK$ is the Hooke's constant of the spring.
Here $V(x)$ is called the \emph{potential} energy in existing literature on physics.

\begin{minipage}{.33\textwidth}
	\begin{figure}[H]
		\centering
		\includegraphics[width=1.3in]{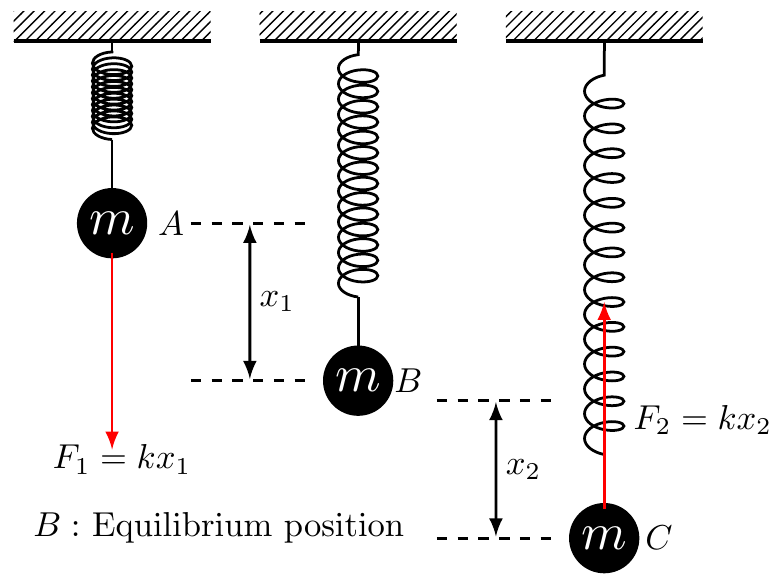}
		\includegraphics[width=1.45in]{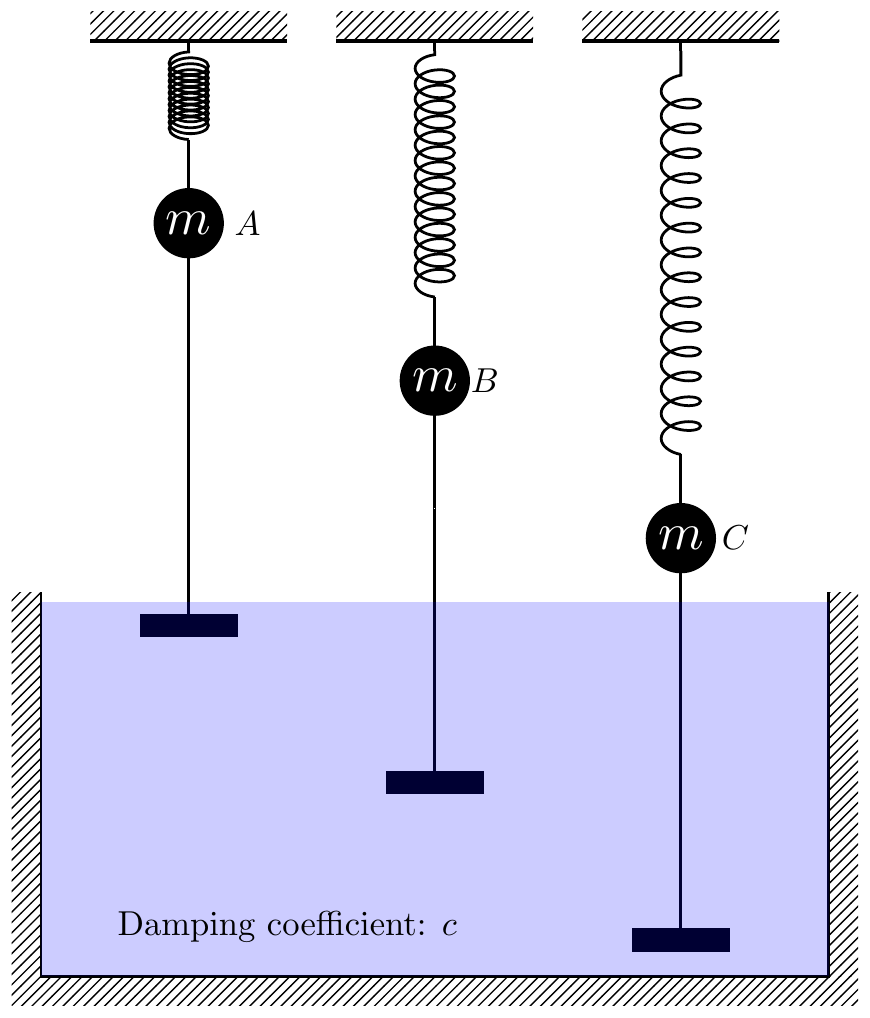}
		\caption{An illustration of the harmonic oscillators: A massive particle connects to a massless spring. (Top) Undamped harmonic oscillator; (Bottom) Damped harmonic oscillator. }
		\label{DHO-example}
	\end{figure}
\end{minipage}\hspace{0.05in}
\begin{minipage}{.62\textwidth}
	
	When one end of spring is attached to a fixed point, and the other end is attached to a freely moving particle with mass $m$, we obtain a \emph{harmonic oscillator}, as illustrated in Figure \ref{DHO-example}. If there is no friction on the particle, by Newton's law, we write the differential equation to describe the system:
	\begin{align*}
	m\ddot{X} + \cK X = 0
	\end{align*}
	where 
	$\ddot{X}: = d^2X/dt^2$ is 
	the acceleration of the particle.
	If we compress the spring and release it at point $x_0$, the system will start oscillating, i.e., at time $t$, the position of the particle is $X(t) = x_0\cos(\omega t)$, where $\omega=\sqrt{\cK/m}$ is the oscillating frequency. 
	
	Such a system has two physical properties: (1) The total energy
	\begin{align*}
	\cE(t):=V(X(t)) + K(X(t)) = V(x_0)
	\end{align*}
	is always a constant, where $K(X):=\frac{1}{2}m\dot{X}^2$ is the kinetic energy of the system. This is also called \emph{energy conservation} in physics;
	(2) The system never stops.
\end{minipage}
\vspace{0.1in}

The harmonic oscillator is closely related to optimization algorithms. As we will show later, all our aforementioned optimization algorithms simply simulate a system, where a particle is falling inside a given potential. From a perspective of optimization, the equilibrium is essentially the minimizer of the quadratic potential function $V(x) = \frac{1}{2}\cK x^2$. The desired property of the system is to stop the particle at the minimizer. However, a simple harmonic oscillator would not be sufficient and does not correspond to a convergent algorithm, since the system never stops: the particle at the equilibrium has the largest kinetic energy, and the inertia of the massive particle would drive it away from the equilibrium.

One natural way to stop the particle at the equilibrium is adding damping to the system, which dissipates the mechanic energy, just like the real-world mechanics. 
A simple damping is a force proportional to the negative velocity of the particle (e.g. submerge the system in some viscous fluid) defined as
\begin{align*}
F_f= -c \dot{X},
\end{align*}
where $c$ is the \emph{viscous damping coefficient}. 
Suppose the potential energy of the system is $f(x)$, then the differential equation of the system is,
\begin{align}
\label{eqn:damped oscillator}
m\ddot{X} + c\dot{X} + \nabla f(X) = 0.
\end{align}
For the quadratic potential, i.e., $f(x) = \frac{\cK}{2}\|x-x^*\|^2$, the energy exhibits exponential decay, i.e.,
\begin{align*}
\cE(t)\propto \exp(-ct/(2m))
\end{align*}
for \emph{under damped} or nearly \emph{critical damped system} (e.g. $c^2\lesssim 4m \cK$).

For an \emph{over damped system} (i.e. $c^2>4m\cK$), the energy decay is
\begin{align*}
\cE(t)\propto \exp\Big(-\frac{1}{2}\Big[\frac{c}{m}-\sqrt{\frac{c^2}{m^2}-\frac{4\cK}{m}}\Big]t\Big).
\end{align*}
For extremely over damping cases, i.e., $c^2\gg 4m\cK$, we have $\frac{c}{m}-\sqrt{\frac{c^2}{m^2}-\frac{4\cK}{m}}\rightarrow \frac{2\cK}{c}$. This decay does not depend on the particle mass. 
The system exhibits a behavior as if the particle has no mass. In the language of optimization, the corresponding algorithm has linear convergence.
Note that the convergence rate does only depend on the ratio $c/m$ and does not depend on $\cK$ when the system is under damped or critically damped.
The fastest convergence rate is obtained, when the system is critically damped, $c^2=4m\cK$. 
\subsection{Sufficient Conditions for Convergence}

For notational simplicity, we assume that 
$x^*=0$ is a global minimum of $f$ with $f(x^*)=0$. 
The potential energy of the particle system is simply defined as $V(t):=V(X(t)):= f(X(t))$. If an algorithm converges {to optimal, a sufficient condition is that} the corresponding potential energy $V$ {decreases} over time. 
The decreasing rate determines the convergence rate of the corresponding algorithm.

\begin{theorem}
\label{thm:convergence general}
Let $\gamma(t)>0$ be a nondecreasing function of $t$ and 
$\Gamma(t)\ge 0$ be a nonnegative function. Suppose that $\gamma(t)$ and $\Gamma(t)$ satisfy
\begin{align*}
 \frac{d (\gamma(t)(V(t)+\Gamma(t))) }{dt} \le 0 \quad \text{and}\quad\lim_{t\rightarrow 0^+}\gamma(t)(V(t)+\Gamma(t))) < \infty.
\end{align*}
Then the convergence rate of the algorithm is characterized by $\frac{1}{\gamma(t)}$.
\end{theorem}
\begin{proof} By $\frac{d (\gamma(t)(V(t)+\Gamma(t))) }{dt} \le 0$, we have
\begin{align*}
 \gamma(t)(V(t) + \Gamma(t)) \le \gamma(0^+)(f(X(0^+)) + \Gamma(0^+)).
\end{align*}
This further implies $f(X) \le V(t)  +\Gamma(t) \le \frac{\gamma(0^+) (f(X(0^+)) + \Gamma(0^+))}{\gamma(t)}$.
\end{proof}
{In words,  $\gamma(t)[V(t)+\Gamma(t)]$ serves as a Lyapunov function of system.}
We say that an algorithm is $(1/\gamma)$-\emph{convergent}, if the potential energy decay rate is $\cO(1/\gamma)$. 
For example, $\gamma(t) = e^{a t}$ corresponds to linear convergence, and $\gamma=at$ corresponds to sublinear convergence, where $a$ is a constant and independent of $t$.
In the following section, we apply Theorem~\ref{thm:convergence general} to different problems by choosing different $\gamma$'s and $\Gamma$'s.

%% file: applications.tex
\section{Convergence Rate in Continuous Time}

We derive the convergence rates of different algorithms
for different families of objective functions. Given our proposed framework, we only need to find $\gamma$ and $\Gamma$ to characterize the energy decay.

\subsection{Convergence Analysis of VGD}

We study the convergence of VGD for two classes of functions: {\bf (1)} General convex function --- \cite{Nes13} has shown that VGD achieves $\cO(L/k)$ convergence for general convex functions; {\bf (2)} A class of functions satisfying the Polyak-\L ojasiewicz (P\L) condition, which is defined as follows \citep{Pol63,KNS16}.
\begin{assumption}
We say that $f$ satisfies the $\mu$-P\L~condition, if there exists a constant $\mu$ such that
for any $x\in\RR^d$, we have $0<\frac{f(x)}{\nbr{\nabla f(x)}^2}\le \frac{1}{2\mu}$.
\end{assumption}


\cite{KNS16} has shown that the P\L~condition is the weakest condition among the following conditions: strong convexity (SC), essential strong convexity (ESC), weak strong convexity (WSC), restricted secant inequality (RSI) and error bound (EB). Thus, the convergence analysis for the P\L~condition naturally extends to all the above conditions. Please refer to \cite{KNS16} for more detailed definitions and analyses as well as various examples satisfying such a condition in machine learning.

\subsubsection{Sublinear Convergence for General Convex Function}

By choosing $\Gamma(t) = \frac{c\nbr{X}^2}{2t}$ and $\gamma(t) = t$, we have
{
\begin{align*}
\frac{d(\gamma(t)(V(t)+\Gamma(t)))}{dt} &= f(X(t)) + t\inner{\nabla f(X(t))}{\dot{X}(t)} + \inner{X(t)}{c\dot{X}(t)}\nonumber\\
&=f(X(t)) - \inner{\nabla f(X(t))}{X(t)} - \frac{t}{c}\nbr{\nabla f(X(t))}^2
\le 0,
\end{align*}
}
where the last inequality follows from the convexity of $f$. Thus, Theorem~\ref{thm:convergence general} implies
\begin{align}\label{VGD-GC-rate}
 f(X(t))\leq\frac{c\norm{x_0}^2}{2t}.
\end{align}
Plugging $t=kh$ and $c=h/\eta$ into \eqref{VGD-GC-rate} and set $\eta=\frac{1}{L}$, we match the convergence rate in \cite{Nes13}:
\begin{align}
\label{eqn:VGD-GC-rate-discrete}
 f(x^{(k)})\le \frac{c\nbr{x_0}^2}{2kh}={\frac{L\nbr{x_0}^2}{2k}}.
\end{align}

\subsubsection{Linear Convergence Under the Polyak-\L{}ojasiewicz Condition}\label{sec:VGD-PL}


Equation \eqref{unified-ODE} implies $\dot{X} = -\frac{1}{c}\nabla f(X(t))$. 
By choosing $\Gamma(t) = 0$ and $\gamma(t)= \exp\rbr{\frac{2\mu t}{c}}$, we  obtain
{
\begin{align*}
\frac{d(\gamma(t)(V(t)+\Gamma(t)))}{dt}&=\gamma(t) {\rbr{\frac{2\mu}{c}f(X(t)) +\inner{\nabla f(X(t))}{\dot{X}(t)}}} 
\\
&=\gamma(t){\rbr{\frac{2\mu}{c}f(X(t)) -\frac{1}{c}\nbr{\nabla f(X(t))}^2}}.
\end{align*}
}
By the $\mu$-P\L~condition: $0<\frac{f(X(t))}{\nbr{\nabla f(X(t))}^2}\le \frac{1}{2\mu}$ for some constant $\mu$ and any $t$, we have
\begin{align*}
 \frac{d(\gamma(t)(V(t)+\Gamma(t)))}{dt}\le 0.
\end{align*}
By Theorem~\ref{thm:convergence general}, for some constant $C$ depending on $x_0$, we obtain
\begin{align}\label{VGD-PL-rate}
 f(X(t)) \leq C'\exp\left(-\frac{2\mu t}{c}\right),
\end{align}
which matches the behavior of an extremely over damped harmonic oscillator.
Plugging $t=kh$ and $c=h/\eta$ into \eqref{VGD-PL-rate} and set $\eta=\frac{1}{L}$, we match the convergence rate in \cite{KNS16}:
\begin{align}
\label{eqn:VGD-convergence}
 f(x_k) \le {C}\exp\left(-\frac{2\mu}{L}k\right)
\end{align}
for some  constant $C$ depending on $x^{(0)}$.

\subsection{Convergence Analysis of NAG}

We study the convergence of NAG for a class of convex functions satisfying the Polyak-\L ojasiewicz (P\L) condition.  The convergence of NAG has been studied for general convex functions in \cite{SBC14}, and therefore is omitted. \cite{Nes13} has shown that NAG achieves a linear convergence for strongly convex functions. Our analysis shows that the strong convexity can be relaxed as it does in VGD. However, in contrast to VGD, NAG requires $f$ to be convex.
For a  $L$-smooth convex function satisfying $\mu$-P\L~condition, we have the particle mass and damping coefficient as $m=\frac{h^2}{\eta}\quad\textrm{and}\quad c = \frac{2\sqrt{\mu} h}{\sqrt{\eta}}=2\sqrt{m\mu}$.
By \cite{KNS16}, under convexity, P\L~is equivalent to quadratic growth (QG). 
Formally, we assume that $f$ satisfies the following condition.
\begin{assumption}
We say that $f$ satisfies the $\mu$-QG condition, if there exists a constant $\mu$ such that for any $x\in\RR^d$, we have $ f(x)-f(x^*)\ge \frac{\mu}{2}\nbr{x-x^*}^2$.
\end{assumption}


We then proceed with the proof for NAG. We first define two parameters, $\lambda$ and $ \sigma$. Let
\[
 \gamma(t) = \exp(\lambda c t)\quad\textrm{and}\quad\Gamma(t) = \frac{m}{2}\|\dot{X} + \sigma cX\|^2.
\]

Given properly chosen $\lambda$ and $\sigma$, we show that the required condition in Theorem~\ref{thm:convergence general} is satisfied.  
Recall that our proposed physical system has kinetic energy $\frac{m}{2}\|\dot{X}(t)\|^2$. In contrast to an un-damped system, NAG takes an effective velocity $\dot{X}+{\sigma c}X$ in the viscous fluid. 
By simple manipulation,
{
\begin{align*}
 \frac{d(V(t) + \Gamma(t))}{dt}=\langle{\nabla f(X)},{\dot{X}}\rangle+m\langle{\dot{X} + \sigma c X},{\ddot{X} + \sigma c \dot{X}}\rangle.
\end{align*}
}
We then observe
{
\begin{align*}
\exp (-\lambda c t) &t\frac{d(\gamma(t)(V(t) + \Gamma(t)))}{dt} 
= \Big[\lambda c f(X) + \frac{\lambda c m}{2}\norm{\dot{X} + \sigma c X}^2 + \frac{d(V(t) + \Gamma(t))}{dt}\Big]
\\
&\le \Big[\lambda c\left(1+\frac{m\sigma^2 c^2}{\mu}\right)f(X)+
\langle{\dot{X}},\left(\frac{\lambda c m}{2} + m\sigma c\right)\dot{X}
+ \nabla f(X)+ m\ddot{X}\rangle
\nonumber\\
& ~~~~   + \langle{ X},{(\lambda \sigma m c^2 + m\sigma^2 c^2)\dot{X} + m\sigma c \ddot{X} }\rangle
\Big].
\end{align*}
}
Since $c^2= 4m\mu$, we argue that if  positive $\sigma$ and $\lambda$ satisfy 
\begin{align}
\label{NAG-QG-condition}
 m(\lambda + \sigma)=1\quad \textrm{and} \quad \lambda\rbr{1+\frac{m\sigma^2 c^2}{\mu}}\le \sigma,
\end{align}
then we guarantee $\frac{d(\gamma(t)(V(t) + \Gamma(t)))}{dt}\le 0$.  
Indeed, we obtain 
\begin{align*}
& \langle \dot{X}, \left(\frac{\lambda c m}{2} + m\sigma c\right)\dot{X} + \nabla f(X) + m\ddot{X}\rangle=-\frac{\lambda mc}{2}\norm{\dot{X}}^2\le 0
\quad\textrm{and}
\\
& \langle X, (\lambda \sigma m c^2 + m\sigma^2 c^2)\dot{X} + m\sigma c \ddot{X} \rangle = -\sigma c \langle{ X},{\nabla f(X)}\rangle.
\end{align*}
By convexity of $f$, we have $\lambda c\big(1+\frac{m\sigma^2 c^2}{\mu}\big)f(X) - \sigma c \langle X, \nabla f(X)\rangle\le \sigma c f(X) - \sigma c \langle X, \nabla f(X)\rangle \le 0$. To make \eqref{NAG-QG-condition} hold, it is sufficient to set $\sigma = \frac{4}{5m}$ and $\lambda = \frac{1}{5m}$.
By Theorem~\ref{thm:convergence general}, we obtain
\begin{align}\label{NAG-convergence}
 f(X(t)) \leq C''\exp \left(-\frac{ct}{5m}\right)
\end{align}
for some constant $C''$ depending on $x^{(0)}$.
Plugging $t=hk$, $m=\frac{h^2}{\eta}$, $c=2\sqrt{m\mu}$, and $\eta=\frac{1}{L}$ into \eqref{NAG-convergence}, we have that
{
\begin{align}
\label{NAG-QG-rate} f(x_k)\le C''\exp\rbr{-\frac{2}{5}\sqrt{\frac{\mu}{L}}k}.
\end{align}
}
Comparing with VGD,  NAG improves the constant term on the convergence rate for convex functions satisfying P\L~condition from $L/\mu$ to $\sqrt{L/\mu}$.
This matches with the algorithmic proof of \cite{Nes13} for strongly convex functions, and \cite{zhang2016new} for convex functions satisfying the QG condition.

\subsection{Convergence Analysis of RCGD and ARCG}

Our proposed framework also justifies the convergence analysis of the RCGD and ARCG algorithms. We will show that the trajectory of the RCGD algorithm converges weakly to the VGD algorithm, and thus our analysis for VGD directly applies.
Conditioning on $x^{(k)}$, the updating formula for RCGD is
\begin{align}
\label{CD-update-vector}
 x^{(k)}_{i} = x^{(k-1)}_{i} - \eta \nabla_i f(x^{(k-1)})
\quad\textrm{and}\quad x^{(k)}_{\setminus i}=x^{(k-1)}_{\setminus i},
\end{align}
where $\eta$ is the step size and $i$ is randomly selected from $\{1, 2,\ldots, d\}$ with equal probabilities. 
Fixing a coordinate $i$, we compute its expectation and variance as
\begin{align*}
\EE\big(x_i^{(k)} -  {x}_i^{(k-1)}\big| x_i^{(k)}\big) 
&
=-\frac{\eta}{d}\nabla_if\rbr{x^{(k-1)}}
{~\text{and}~}\\
\Var\big(x_i^{(k)} -  {x}_i^{(k-1)}\big| x_i^{(k)}
\big)
&
= \frac{\eta^2(d-1)}{d^2}{\nbr{\nabla_if\rbr{x^{\rbr{k-1}}}}^2}.
\end{align*}
We define the infinitesimal time scaling factor $h\le \eta$ as it does in Section~\ref{sec:unified framework} and denote $\tilde{X}^h(t) := x^{(\lfloor t/h\rfloor)}$. 
We prove that for each $i\in[d]$, $\tilde{X}_i^h(t)$ converges weakly to a deterministic function $X_i(t)$ as $\eta\rightarrow 0$. 
Specifically, we rewrite \eqref{CD-update-vector} as,
\begin{align}
\label{CD-update-vector-inf}
 \tilde{X}^h(t+h)-\tilde{X}^h(t) = -\eta \nabla_i f(\tilde{X}^h(t)).
\end{align}
Taking the limit of $\eta\rightarrow 0$ at a fix time $t$, we have
\begin{align*}
\abr{X_i(t+h)-X_i(t)} = \cO(\eta)~\textrm{and}~\frac{1}{\eta}\EE\big(\tilde{X}^h(t+h) -  \tilde{X}^h(t)\big| \tilde{X}^h(t)\big)=-\frac{1}{d}\nabla f(\tilde{X}^h(t)) + \cO(h).
\end{align*}
Since $\norm{\nabla f(\tilde{X}^h(t))}^2$ is bounded at the time $t$, we have $\frac{1}{\eta}\Var\big(\tilde{X}^h(t+h) -  \tilde{X}^h(t)\big| \tilde{X}^h(t)\big)=\cO(h)$. Using an infinitesimal generator argument in \cite{ethier2009markov}, we conclude that $\tilde{X}^h(t)$ converges to $X(t)$ weakly as $h\rightarrow 0$, where $X(t)$ satisfies, $\dot{X}(t) + \frac{1}{d}\nabla f(X(t)) = 0$ and $X(0)=x^{(0)}$. Since $\eta\le \frac{1}{L_{\max}}$, by \eqref{eqn:VGD-convergence}, we have
\[
 f(x_k)\le C_1\exp\big(-\frac{2\mu}{dL_{\max}}k\big).
\]
for some  constant $C_1$ depending on $x^{(0)}$.
The analysis for general convex functions follows similarly.
One can easily match the convergence rate as it does in \eqref{eqn:VGD-GC-rate-discrete}, $ f(x^{(k)})\le \frac{c\nbr{x_0}^2}{2kh}={\frac{dL_{\max}\nbr{x_0}^2}{2k}}$.

Repeating the above argument for ARCG, we obtain that the trajectory $\tilde{X}^h(t)$ converges weakly to $X(t)$, where $X(t)$ satisfies
\[
 m\ddot{X}(t) + c\dot{X}(t) + \nabla f(X(t)) = 0.
\]
For general convex function, we have $ m=\frac{h^2}{\eta'}$ and $c= \frac{3m}{t}$, where $\eta'=\frac{\eta}{d}$.
By the analysis of  \cite{SBC14}, we have $f(x_k)\le \frac{C_2 d}{k^2}$, for some constant $C_2$ depending on $x^{(0)}$ and $L_{\max}$.

For convex functions satisfying $\mu$-QG condition, $m=\frac{h^2}{\eta'}$ and $c= 2\sqrt{\frac{m\mu}{d}}$. By \eqref{NAG-QG-rate}, we obtain $f(x_k)\le C_3\exp\big(-\frac{2}{5d}\sqrt{\frac{\mu}{L_{\max}}}\big)$ for some constant $C_3$ depending on $x^{(0)}$.

\subsection{Convergence Analysis for Newton}

Newton's algorithm is a second-order algorithm. Although it is different from both VGD and NAG, we can fit it into our proposed framework by choosing $\eta=\frac{1}{L}$ and the gradient as $L\sbr{\nabla^2f(X)}^{-1}\nabla f(X)$.
We consider only the case $f$ is $\mu$-strongly convex, $L$-smooth and $\nu$-self-concordant.
By \eqref{unified-ODE}, if $h/\eta$ is not vanishing under the limit of $h\rightarrow 0$, we achieve a similar equation,
\[
  \bC\dot{X} + \nabla f(X) = 0,
\]
where $\bC=h\nabla^2 f(X)$ is the \emph{viscosity tensor} of the system.
In such a system, the function $f$ not only determines the gradient field, but also determines a viscosity tensor field.
The particle system is as if submerged in an anisotropic fluid that exhibits different viscosity along different directions. 
We release the particle at point $x_0$ that is sufficiently close to the minimizer $0$, i.e. $\nbr{x_0-0}\le \zeta$ for some parameter $\zeta$ determined by $\nu$, $\mu$, and $L$. 
Now we consider the decay of the potential energy $V(X):=f(X)$.
By Theorem~\ref{thm:convergence general} with $\gamma(t) = \exp(\frac{t}{2h})$ and $\Gamma(t) = 0$, we have
{
\begin{align*}
&  \frac{d(\gamma(t)f(X))}{dt} = \exp\rbr{\frac{t}{2h}} \cdot
\yestwocl{
\\
&~~~~
}
  \sbr{\frac{1}{2h} f(X) - \frac{1}{h}\inner{\nabla f(X)}{ (\nabla^2f(X))^{-1}\nabla f(X)}}.
\end{align*}
}
By simple calculus, we have $  \nabla f(X) = -{\int_{1}^0\nabla^2f((1-t)X)dt}\cdot{X}$. By the self-concordance condition, we have
\begin{align*}
  {(1-\nu t\nbr{X}_X)^2}\nabla^2f(X)
\yestwocl{&}  \preceq \nabla^2f((1-t)X)dt 
\yestwocl{\\&}  \preceq\frac{1}{(1-\nu t\nbr{X}_X)^2}\nabla^2f(X),
\end{align*}
where $\nbr{v}_X = \rbr{v^T\nabla^2 f(X) v}\in \sbr{\mu\nbr{v}_2, L\nbr{v}_2}$. Let $\beta =\nu \zeta L \le 1/2$. By integration and the convexity of $f$, we have
{

\begin{align*}
  {(1-\beta)}\nabla^2f(X) &  \preceq \int_{0}^1\nabla^2f((1-t)X)dt \preceq\frac{1}{1-\beta}\nabla^2f(X)
\\
\textrm{and}~~
  \frac12 f(X) &  - \inner{\nabla f(X)}{ (\nabla^2f(X))^{-1}\nabla f(X)} 
\yestwocl{
\\&
}
  \le \frac12 f(X) - \frac{1}{2}\inner{\nabla f(X)}{X} 
\le 0.
\end{align*}
}

Note that our proposed ODE framework only proves a local linear convergence for Newton method under the strongly convex, smooth and self concordant conditions. The convergence rate contains an absolute constant, which does not depend on $\mu$ and $L$. This partially justifies the superior local convergence performance of the Newton's algorithm for ill-conditioned problems with very small $\mu$ and very large $L$. Existing literature, however, has proved the local quadratic convergence of the Newton's algorithm, which is better than our ODE-type analysis. This is mainly because the discrete algorithmic analysis takes the advantage of ``large'' step sizes, but the ODE only characterizes ``small'' step sizes, and therefore fails to achieve quadratic convergence.

%% file: extension.tex

\vspace{-0.175in}

%% file: discussion.tex

\section{Numerical Illustration and Discussions}

Due to the space limit, we present an numerical illustration in Figure 2. See more details on numerical results in Appendix \ref{numerical}. 
We then give a more detailed interpretation of our proposed system from a perspective of physics:

\noindent{\bf Consequence of Particle Mass} --- As  shown in Section \ref{sec:phys}, a massless particle system (mass $m=0$) describes the simple gradient descent algorithm.

\begin{minipage}{.3\textwidth}
\begin{figure}[H]
\centering
\includegraphics[width=1.5in]{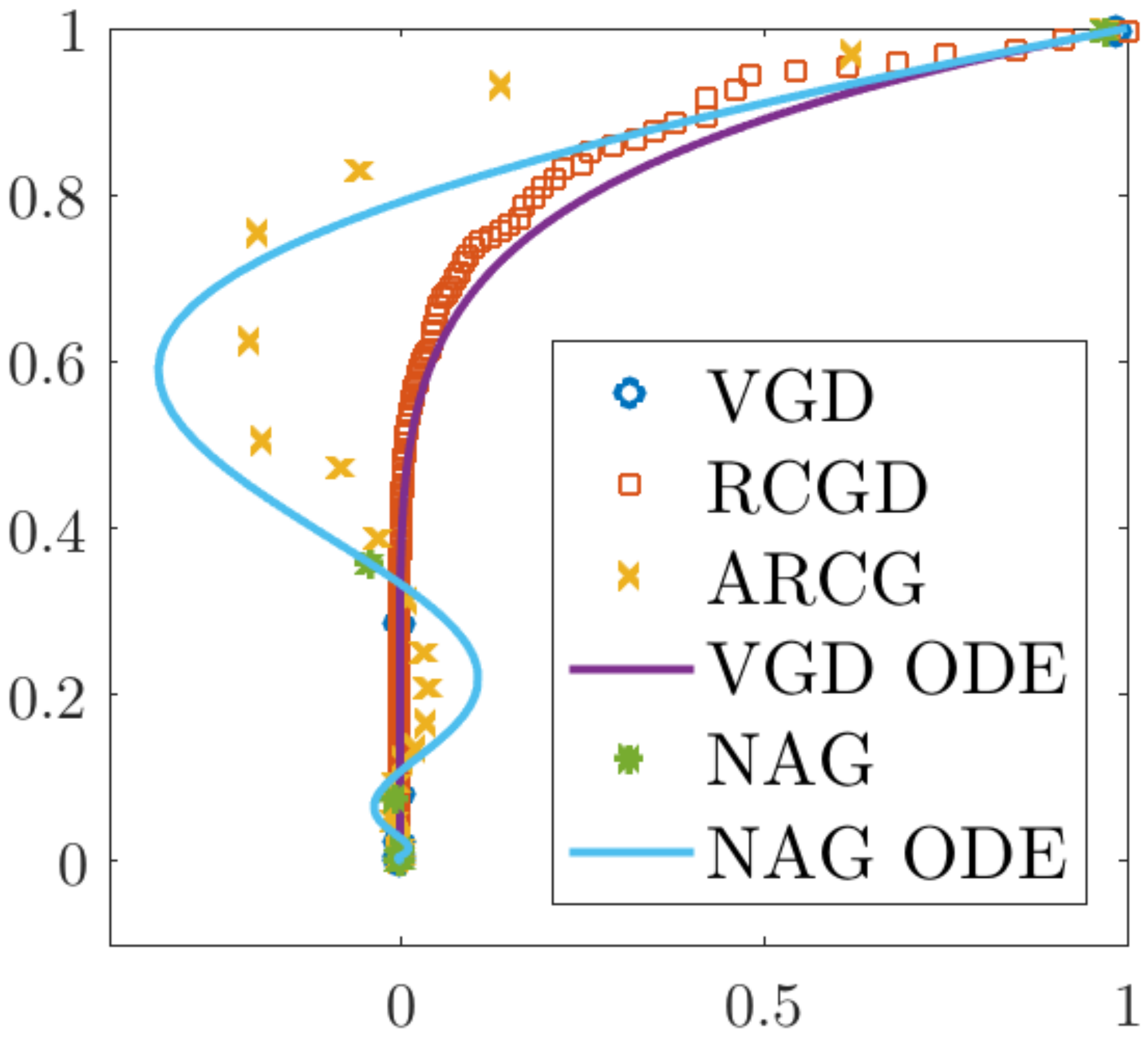}\label{trajectory}
\caption{\small The algorithmic iterates and trajectories of a simple quadratic program.}
\end{figure}
\vspace{0.01in}
\end{minipage}\hspace{0.1in}
\begin{minipage}{.64\textwidth}
By Newton's law, a $0$-mass particle can achieve infinite acceleration and has infinitesimal response time to any force acting on it. Thus, the particle is ``locked'' on the force field (the gradient field) of the potential ($f$) -- the velocity of the particle is always proportional to the restoration force acting on the particle.
The convergence rate of the algorithm is only determined by the function $f$ and the damping coefficient. The mechanic energy is  stored in the force field (the potential energy) rather than in the kinetic energy. 
Whereas for a massive particle system, the mechanic energy is also partially stored in the kinetic energy of the particle. Therefore, even when the force field is not strong enough, the particle keeps a high speed.
\end{minipage}

\vspace{0.1in}

\noindent{\bf Damping and Convergence Rate}  --- For a quadratic potential $V(x) = \frac{\mu}{2}\nbr{x}^2$, the system has a exponential energy decay, where the exponent factor depends on mass $m$, damping coefficient $c$, and the property of the function (e.g. P{\L}-conefficient). As discussed in Section~\ref{sec:phys}, the decay rate is the fastest when the system is critically damped, i.e, $c^2=4m\mu$. For either under or over damped system, the decay rate is slower. For a potential function $f$ satisfying convexity and $\mu$-P\L{} condition, NAG corresponds to a nearly critically damped system, whereas VGD corresponds to an extremely over damped system, i.e., $c^2\gg 4m\mu$. Moreover, we can achieve different acceleration rate by choosing different $m/c$ ratio for NAG, i.e., $\alpha=\frac{{1/(\mu\eta)}^{s}-1}{{1/(\mu\eta)}^{s}+1}$ for some absolute constant $s>0$. However $s=1/2$ achieves the largest convergence rate since it is exactly the \emph{critical damping}: $c^2=4{m\mu}$.


\noindent{\bf Connecting P\L{} Condition to Hooke's law} --- The $\mu$-P\L{} and convex conditions together naturally mimic the property of a quadratic potential $V$, i.e., a damped harmonic oscillator. Specifically, the $\mu$-P\L{} condition

\notwocl{
\begin{figure}[H]
\centering
\includegraphics[width=4.2in]{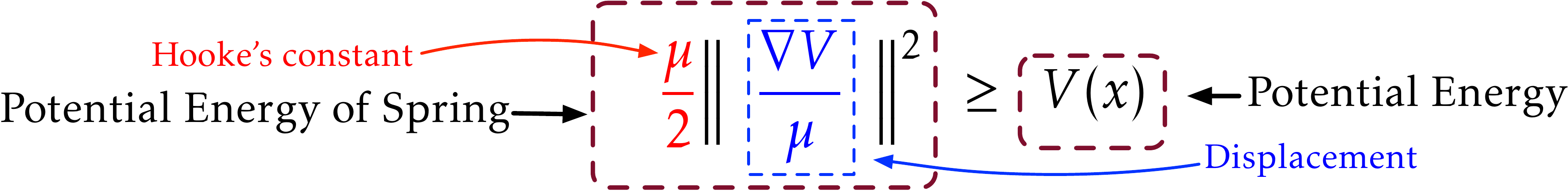}
\end{figure}
}
\noindent guarantees that the force field is strong enough, since the left hand side of the above equation 
is exactly the potential energy of a spring based on Hooke's law. Moreover, the convexity condition $V(x) \le \inner{\nabla V(x)}{X}$ guarantees that the force field has a large component pointing at the equilibrium point (acting as a restoration force).
{
As indicated in \cite{KNS16}, P\L~is a much weaker condition than the strong convexity. Some functions that satisfy local P\L~condition do not even satisfy convexity, e.g., matrix factorization.
The connection between the P\L~condition and the Hooke’s law indicates that strong convexity is not the fundamental characterization of linear convergence. If there is another condition that employs a form of the Hooke’s law, it should employ linear convergence as well.}

%% file: appendix.tex
\section{A Brief Review of Popular Optimization Algorithms}
\label{appendix:optimization algs}
\subsection{Vanilla Gradient Descent Algorithm} 
A vanilla gradient descent (VGD) algorithm starts from an arbitrary initial
solution $x^{(0)}$. At the $k$-th iteration ($k>0$), VGD takes
\begin{align*}
x^{(k)} = x^{(k-1)} - \eta \nabla f(x^{(k-1)}),
\end{align*}
where $\eta$ is a properly chosen step size. Since VGD only needs to calculate a gradient of $f$ in each iteration, the computational cost per iteration is usually linearly dependent on $d$. For a $L$-smooth $f$, we can choose a constant step size such that $\eta\leq\frac{1}{L}$ to guarantee convergence. 

VGD has been extensively studied in existing literature. \cite{Nes13} show that:
\begin{description}
\item (1) For general convex function, VGD attains a sublinear convergence rate as
\begin{align}\label{VGD-sublinear}
f(x^{(k)}) - f(x^*) \leq \frac{L\norm{x^{(0)}-x^*}^2}{2k}\quad\textrm{for}~k=1,2,.....
\end{align}
Note that \eqref{VGD-sublinear} is also referred as an iteration complexity of $\cO(L/\epsilon)$, i.e., we need $\cO(L/\epsilon)$ such that $f(x^{(k)}) - f(x^*)\leq\epsilon$, where $\epsilon$ is a pre-specified accuracy of the objective value.
\item (2) For a $L$-smooth and $\mu$-strongly convex $f$, VGD attains a linear convergence rate as
\begin{align}\label{VGD-linear}
f(x^{(k)}) - f(x^*) \leq \left(1-\frac{1}{\kappa}\right)^k\frac{L\norm{x^{(0)}-x^*}^2}{2}
\ifdefined\twocl
\nonumber\\
\else
\quad
\fi
\textrm{for}~k=1,2,.....
\end{align}
Note that \eqref{VGD-linear} is also referred as an iteration complexity of $\cO(\kappa\cdot\log(1/\epsilon))$.
\end{description}

\subsection{Nesterov's Accelerated Gradient Algorithms}

The Nesterov's accelerated gradient (NAG) algorithms combines the vanilla gradient descent algorithm with an additional momentum at each iteration. Such a modification, though simple, enables NAG to attain better convergence rate than VGD. Specifically, NAG starts from an arbitrary initial
solution $x^{(0)}$ along with an auxiliary solution $y^{(0)}=x^{(0)}$.
At the $k$-th iteration, NAG takes
\begin{align*}
x^{(k)} = y^{(k-1)} - \eta \nabla f(y^{(k-1)})
\ifdefined\twocl
\nonumber\\
\else
\quad\textrm{and}\quad
\fi
 y^{(k)} = x^{(k)}+\alpha(x^{(k)} -x^{(k-1)}),
\end{align*} 
where $\alpha=\frac{k-1}{k+2}$ for general convex $f$ and $\alpha=\frac{\sqrt{\kappa}-1}{\sqrt{\kappa}+1}$ for strongly convex $f$. Intuitively speaking, NAG takes an affine combination 
of the current and previous solutions to compute the update for the two subsequent iterations. This can be viewed as the momentum of a particle during its movement. Similar to VGD, NAG only needs to calculate a gradient of $f$ in each iteration. Similar to VGD, we can choose $\eta\leq\frac{1}{L}$ for a $L$-smooth $f$ to guarantee convergence. 

NAG has also been extensively studied in existing literature. \cite{Nes13} show that:
\begin{description}
\item (1) For general convex function, NAG attains a sublinear convergence rate as
\begin{align}\label{NAG-sublinear}
f(x^{(k)}) - f(x^*) \leq \frac{2L\norm{x^{(0)}-x^*}^2}{k^2}
\quad\textrm{for}~k=1,2,.....
\end{align}
Note that \eqref{NAG-sublinear} is also referred as an iteration complexity of $\cO(\sqrt{L/\epsilon})$.
\item (2) For a $L$-smooth and $\mu$-strongly convex $f$, NAG attains a linear convergence rate as
\begin{align}\label{NAG-linear}
f(x^{(k)}) - f(x^*) \leq \left(1-\sqrt{\frac{1}{4\kappa}}\right)^k\frac{L\norm{x^{(0)}-x^*}^2}{2}
\ifdefined\twocl
\nonumber\\
\else
\quad
\fi
\textrm{for}~k=1,2,.....
\end{align}
Note that \eqref{NAG-linear} is also referred as an iteration complexity of $\cO(\sqrt{\kappa}\cdot\log(1/\epsilon))$.
\end{description}

\subsection{Randomized Coordinate Gradient Descent Algorithm} 

A randomized coordinate gradient descent (RCGD) algorithm is closely related to VGD. RCGD starts from an arbitrary initial
solution $x^{(0)}$. Different from VGD, RCGD takes a gradient descent step only over a coordinate. Specifically, at the $k$-th iteration ($k>0$), RCGD randomly selects a coordinate $j$ from $1,...,d$, and takes
\begin{align*}
x^{(k)}_{j} = x^{(k-1)}_{j} - \eta \nabla_j f(x^{(k-1)})
\quad\textrm{and}\quad
x^{(k)}_{\setminus j}=x^{(k-1)}_{\setminus j}.
\end{align*}
where $\eta$ is a properly chosen step size. Since RCGD only needs to calculate a coordinate gradient of $f$ in each iteration, the computational cost per iteration usually does not scale with $d$. For a $L_{\max}$-coordinate-smooth $f$, we can choose a constant step size such that $\eta\leq\frac{1}{L_{\max}}$ to guarantee convergence. 

RCGD has been extensively studied in existing literature. \cite{nesterov2012efficiency,lu2015complexity} show that:
\begin{description}
\item (1) For general convex function, RCGD attains a sublinear convergence rate in terms of the expected objective value as
\begin{align}\label{RCGD-sublinear}
\EE f(x^{(k)}) - f(x^*) \leq \frac{dL_{\max}\norm{x^{(0)}-x^*}^2}{2k}\quad\textrm{for}~k=1,2,.....
\end{align}
Note that \eqref{RCGD-sublinear} is also referred as an iteration complexity of $\cO(dL_{\max}/\epsilon)$.
\item (2) For a $L_{\max}$-smooth and $\mu$-strongly convex $f$, RCGD attains a linear convergence rate in terms of the expected objective value as
\begin{align}\label{RCGD-linear}
\EE f(x^{(k)}) - f(x^*) \leq \left(1-\frac{\mu}{dL_{\max}}\right)^k\frac{L\norm{x^{(0)}-x^*}^2}{2}\ifdefined\twocl
\nonumber\\
\else
\quad
\fi
\textrm{for}~k=1,2,.....
\end{align}
Note that \eqref{RCGD-linear} is also referred as an iteration complexity of $\cO(dL_{\max}/\mu\cdot\log(1/\epsilon))$.
\end{description}

\subsection{Accelerated Randomized Coordinate Gradient Algorithms} 

Similar to NAG, the accelerated randomized coordinate gradient (ARCG) algorithms combine the randomized coordinate gradient descent algorithm with an additional momentum at each iteration. Such a modification also enables ARCG to attain better convergence rate than RCGD. Specifically, ARCG starts from an arbitrary initial
solution $x^{(0)}$ along with an auxiliary solution $y^{(0)}=x^{(0)}$.
At the $k$-th iteration ($k>0$), ARCG randomly selects a coordinate $j$ from $1,...,d$, and takes
\begin{align*}
x_j^{(k)}=y_j^{(k-1)} - \eta\nabla_j f(y^{(k-1)}),\quad x_{\setminus j}^{(k)}=y_{\setminus j}^{(k-1)},
\ifdefined\twocl
\nonumber\\
\else
\quad\text{and}\quad 
\fi
y^{(k)} = x^{(k)}+ \alpha\rbr{x^{(k)}-x^{(k-1)}}.
\end{align*}
Here $\alpha= \frac{\sqrt{\kappa_{\max}}-1}{\sqrt{\kappa_{\max}}+1}$ when $f$ is strongly convex, and $\alpha=\frac{k-1}{k+2}$ when $f$ is general convex.
Similar to RCGD, we can choose $\eta\leq\frac{1}{L_{\max}}$ for a $L_{\max}$-coordinate-smooth $f$ to guarantee convergence.

ARCG has been studied in existing literature. \cite{lin2014accelerated,fercoq2015accelerated} show that:
\begin{description}
\item (1) For general convex function, ARCG attains a sublinear convergence rate in terms of the expected objective value as
\begin{align}\label{ARCG-sublinear}
\EE f(x^{(k)}) - f(x^*) \leq \frac{2d\sqrt{L_{\max}}\norm{x^{(0)}-x^*}^2}{k^2}
\ifdefined\twocl
\nonumber\\
\else
\quad
\fi
\textrm{for}~k=1,2,.....
\end{align}
Note that \eqref{ARCG-sublinear} is also referred as an iteration complexity of $\cO(d\sqrt{L_{\max}}/\sqrt{\epsilon})$.
\item (2) For a $L_{\max}$-smooth and $\mu$-strongly convex $f$, ARCG attains a linear convergence rate in terms of the expected objective value as
\begin{align}\label{ARCG-linear}
\EE f(x^{(k)}) - f(x^*) \leq \left(1-\frac{1}{d}\sqrt{\frac{\mu}{L_{\max}}}\right)^k\frac{L\norm{x^{(0)}-x^*}^2}{2}
\ifdefined\twocl
\nonumber\\
\else
\quad
\fi
\textrm{for}~k=1,2,.....
\end{align}
Note that \eqref{ARCG-linear} is also referred as an iteration complexity of $\cO(d\sqrt{L_{\max}/\mu}\cdot\log(1/\epsilon))$.
\end{description}

\subsection{Newton's Algorithm}

The Newton's (Newton) algorithm requires $f$ to be twice differentiable. It starts with an arbitrary initial $x^{(0)}$. At the $k$-th iteration ($k>0$), Newton takes
\begin{align*}
x^{(k)} = x^{(k-1)} - \eta(\nabla^2f(x^{(k-1)}))^{-1}\nabla f(x^{(k-1)}).
\end{align*}
The inverse of the Hessian matrix adjusts the descent direction by the landscape at $x^{(k-1)}$. Therefore, Newton often leads to a steeper descent than VGD and NAG in each iteration, espcially for highly ill-conditioned problems. 

Newton has been extensively studied in existing literature with an additional self-concordant assumption as follows:
\begin{assumption}
Suppose that $f$ is smooth and convex. We define $g(t)=f(x+tv)$. We say that $f$ is self-concordant, if for any $x\in\RR^d$, $v\in\RR^d$, and $t\in\RR$, there exists a constant $\nu$, which is independent on $f$ such that we have
\begin{align*}
|g'''(t)|\leq \nu g''(t)^{3/2}.
\end{align*}
\end{assumption}
\cite{nocedal2006numerical} show that for a $L$-smooth, $\mu$-strongly convex and $\nu$-self-concordant, $f$, Newton attains a local quadratic convergence in conjunction. Specifically, given a suitable initial solution $x^{(0)}$ satisfying $\norm{x^{(0)}-x^*}_2\leq\zeta$, where $\zeta<1$ is a constant depending on on $L$, $\mu$, and $\nu$, there exists a constant $\xi$ depending only on $\nu$ such that we have
\begin{align}\label{Newton-linear}
f(x_{k+1})-f(x^*)\leq\xi[f(x^{(k)})-f(x^*)]^2
\ifdefined\twocl
\nonumber\\
\else
\quad
\fi
\textrm{for}~k=1,2,.....
\end{align}
 


Note that \eqref{Newton-linear} is also referred as an iteration complexity of $\tilde{\cO}(\log\log(1/\epsilon))$, where $\tilde{\cO}$ hides the constant term depending on $L$, $\mu$, and $\nu$. Since Newton needs to calculate the inverse of the Hessian matrix, its per iteration computation cost is at least $\cO(d^3)$. Thus, it outperforms VGD and NAG when we need a highly accurate solution, i.e., $\epsilon$ is very small.



\section{Extension to Nonsmooth Composite Optimization}

Our framework can also be extended to nonsmooth composite optimization in a similar manner to \cite{SBC14}.
Let $g$ be an $L$-smooth function, and $h$ be a general convex function (not necessarily smooth).
For $x\in\RR^d$, the composite optimization problem solves
\[
\min_{x\in\RR^d} f(x):=g(x) + h(x).
\]
Analogously to \cite{SBC14}, we define the force field as the directional subgradient $G(x, p)$ of function $f$, where $G:\RR^d\times\RR^d\rightarrow \RR^d$ is defined as $G(x, p)\in\partial f(x)$ and $\inner{G(x, p)}{p}=\sup_{\xi\in\partial f(x)}\inner{\xi}{p}$, where $\partial f(x)$ denotes the sub-differential of $f$. The existence of $G(x, p)$ is guaranteed by \cite{Rec15}. Accordingly, a new ODE describing the dynamics of the system is
\[
m\ddot{X} + c\dot{X} + G(X, \dot{X}) = 0.
\]
{Under the assumption that the solution to the ODE exists and is unique,}
we illustrate the analysis by VGD (the mass $m=0$) under the proximal-P\L{} condition. The extensions to other algorithms are straightforward. Specifically, a convex function $f$ satisfies $\mu$-proximal-P\L{} if
\begin{align}\label{proximal-PL-condition}
\frac{1}{2\mu}\inf_{p\in S^{d-1}}\nbr{G(x, p)}^2\ge f(x) - f(x^*), 
\end{align}
where $x^*=0$ is the global minimum point of $f$. Slightly different from the definition of the proximal-P\L{} condition in \cite{KNS16} involving a step size parameter, \eqref{proximal-PL-condition} does not involve any additional parameter. This is actually a more intuitive definition by choosing an appropriate subgradient. Let $\gamma(t) = e^{2\mu t/c}$ and $\Gamma(t) = 0$. For a small enough $\Delta t > 0$, we study $$\frac{{\gamma(t+\Delta t) f(t+\Delta t) - \gamma(t) f(t)}}{\Delta t}.$$

By Taylor expansions and the local Lipschitz property of convex function $f$, we have
\begin{align*}
\gamma(t+\Delta t) &= \exp\rbr{\frac{2\mu t}{c}}\rbr{1+\frac{2\mu}{c}\Delta t} + o(\Delta t)~\textrm{and}~\\
f(X(t+\Delta t)) &= f(X) + \langle{\dot{X}},{G(X, \dot{X})}\langle\Delta t + o(\Delta t)\rangle.
\end{align*}
Combining the above two expansions, we obtain
\begin{align*}
\yestwocl{&} \gamma(t+\Delta t) f(X(t+\Delta t)) 
\yestwocl{\\&}  = \exp\rbr{\frac{2\mu t}{c}}\big({f(X) + \frac{2\mu}{c}f(X)\Delta t + \langle{\dot{X}},{G(X, \dot{X})}\rangle\Delta t}\big) 
\yestwocl{
	\\
	& ~~~~
}
+ o(\Delta t).
\end{align*}
This further implies
{\small
	\begin{align*}
	\yestwocl{&} \frac{{\gamma(t+\Delta t) f(t+\Delta t) - \gamma(t) f(t)}}{\Delta t} 
	\yestwocl{\\&} 
	= \exp\rbr{\frac{2\mu t}{c}}\big({\frac{2\mu}{c}f(X) - \frac{1}{c}\norm{G(X, \dot{X})}^2}\big) + O(\Delta t).
	\end{align*}
}
By the $\mu$-proximal-P\L{} condition of $f$, we have $\lim_{\Delta t\rightarrow0}\frac{{\gamma(t+\Delta t) f(t+\Delta t) - \gamma(t) f(t)}}{\Delta t}\le 0$. The rest of the analysis follows exactly the same as it does in Section~\ref{sec:VGD-PL}.

\section{Numerical Illustration}\label{numerical}
We present an illustration of our theoretical analysis in Figure 2. We consider a strongly convex quadratic program
\begin{align*}
f(x) = \frac{1}{2}x^\top H x,
\quad \text{where} \quad
H = \left[
\begin{array}{cc}
300 &1\\
1 &50\\
\end{array}
\right].
\end{align*}
Obviously, $f(x)$ is strongly convex and $x^*=[0,0]^\top$ is the minimizer. We choose $\eta = 10^{-4}$ for VGD and NAG, and $\eta=2\times 10^{-4}$ for RCGD and ARCG. The trajectories of VGD and NAG are obtained by the default method for solving ODE in MATLAB.